\documentclass[sigconf, nonacm]{acmart}
\AtBeginDocument{%
  }

\setcopyright{acmlicensed}
\copyrightyear{2026}
\acmYear{2026}
\acmDOI{XXXXXXX.XXXXXXX}
\acmConference[MM '26]{Make sure to enter the correct
  conference title from your rights confirmation email}{10–14 November 2026}{Rio de Janeiro, Brazil}
\acmISBN{978-1-4503-XXXX-X/2018/06}

\newtheorem{definition}{Definition}
\newtheorem{theorem}{Theorem}
\usepackage{hyperref}
\usepackage{natbib}
\usepackage{pgfplots}
\usepackage{xcolor}

\definecolor{linkpink}{RGB}{255,20,147}
\definecolor{emailblue}{RGB}{0,0,255}
\AtEndPreamble{%
  \hypersetup{
    colorlinks=true,
    linkcolor=linkpink,
    citecolor=linkpink,
    urlcolor=linkpink,
    filecolor=linkpink
  }%
}

\usepackage{siunitx}
\newtheorem{innercustomthm}{Theorem}
\newenvironment{customthm}[1]
  {\renewcommand\theinnercustomthm{#1}\innercustomthm}
  {\endinnercustomthm}
\begin{document}
\title{Eulerian Motion Guidance: Robust Image Animation via Bidirectional Geometric Consistency}

\author{%
Thong Nguyen\textsuperscript{*} \quad
Khoi M. Le \quad
Cong-Duy Nguyen \\
Luu Anh Tuan \quad
See-Kiong Ng \quad
Chunyan Miao
}

\affiliation[obeypunctuation=true]{%
  \institution{National University of Singapore}, \country{Singapore}%
}
\affiliation[obeypunctuation=true]{%
  \institution{Nanyang Technological University}, \country{Singapore}%
}
\affiliation[obeypunctuation=true]{%
  \institution{Centre for AI Research, VinUniversity}, \country{Vietnam}%
}

\email{thong.nguyen@u.nus.edu}

\renewcommand{\shortauthors}{Nguyen et al.}


\begin{abstract}
Recent advancements in image animation have utilized diffusion models to breathe life into static images. However, existing controllable frameworks typically rely on Lagrangian motion guidance, where optical flow is estimated relative to the initial frame. This paper revisits the same optical-flow primitive through a more local supervision design: we use adjacent-frame Eulerian motion fields to guide generation, where the motion signal always describes a short temporal hop. This shift enables parallelized training and provides bounded-error supervision throughout the generation process. To mitigate the drift artifacts common in adjacent frame generation, we introduce a Bidirectional Geometric Consistency mechanism, which computes a forward-backward cycle check to mathematically identify and mask occluded regions, preventing the model from learning incorrect warping objectives. Extensive experiments demonstrate that our approach accelerates training, preserves temporal coherence, and reduces dynamic artifacts compared to reference-based baselines. The code, model, and data have been made available at \href{https://nguyentthong.github.io/eulerian/}{nguyentthong.github.io/eulerian}.
\end{abstract}




%% acmart typesets the author email via \nolinkurl (no hyperlink, so hyperref
%% colors don't apply). Temporarily wrap \nolinkurl to color the email blue.
\begingroup
\let\ACMorigNoLinkURL\nolinkurl
\renewcommand{\nolinkurl}[1]{\textcolor{emailblue}{\ACMorigNoLinkURL{#1}}}
\maketitle
\endgroup

\begingroup
\renewcommand{\thefootnote}{*}
\footnotetext{Corresponding author.}
\endgroup

\section{Introduction}
\begin{figure*}
    \centering
    \includegraphics[width=0.7\linewidth]{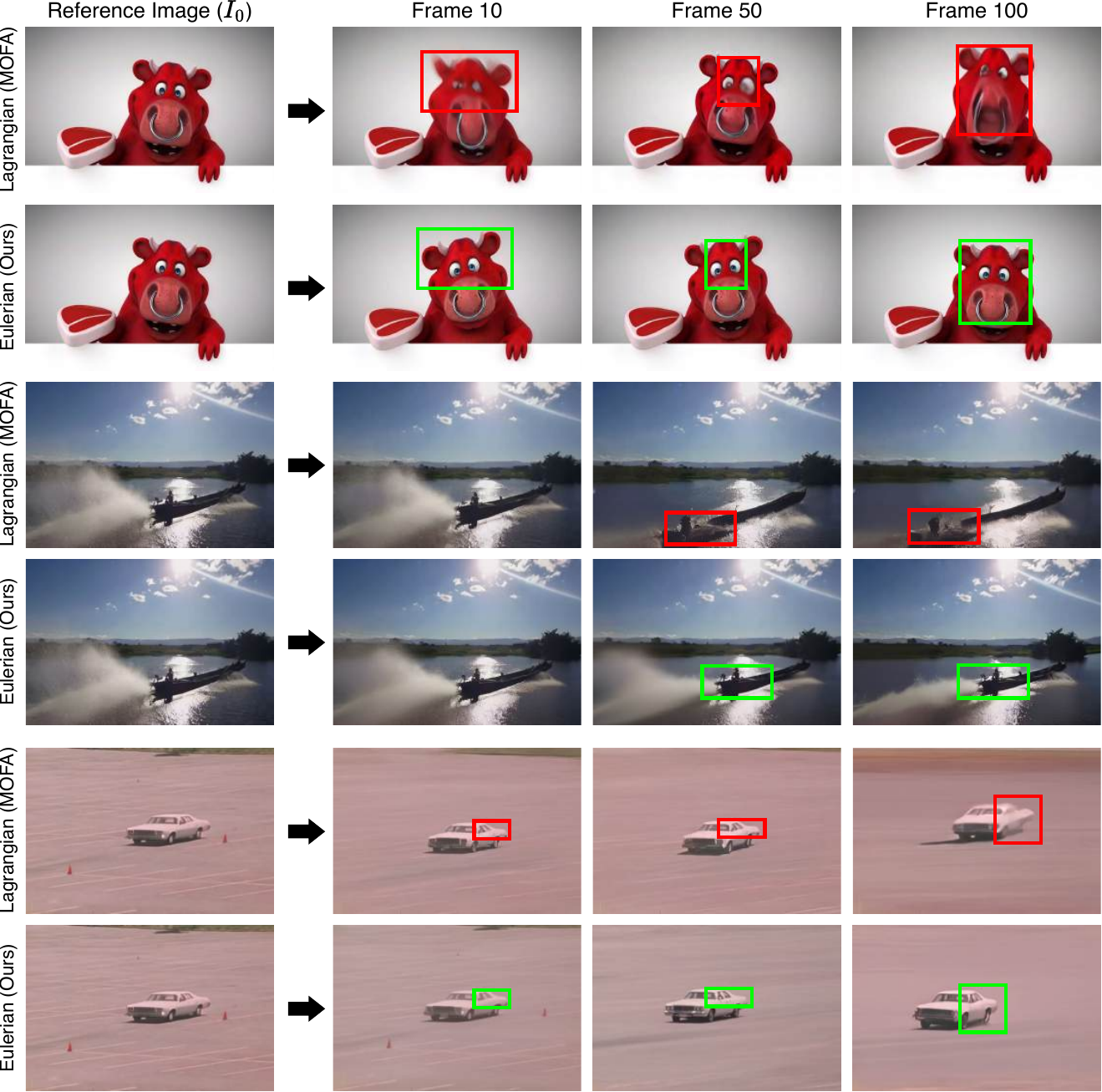}
    \caption{Long-horizon qualitative comparison. We show the reference image and generated frames at $t = \{10, 50, 100\}$. The Lagrangian baseline MOFA \citep{niu2024mofa} exhibits increasing texture drift or blur as temporal distance grows, whereas our Eulerian motion framework better preserves identity and geometry over time.}
    \label{fig:qualitative_comparison_mofa}
\end{figure*}
The pursuit of ``\textit{bringing images to life}'' has long been a fascination in computer vision, evolving from early Generative Adversarial Network (GAN)-based approaches \citep{vondrick2016generating} to recent controllable diffusion models \citep{hu2024animate, zheng2023layoutdiffusion}. While current Image-to-Video (I2V) frameworks such as Stable Video Diffusion \citep{blattmann2023stable} and Lumiere \citep{bar2024lumiere} can generate plausible short clips, ensuring robust long-term temporal consistency under complex motion remains a formidable challenge. The core difficulty lies in maintaining the identity and structure of the initial image while adhering to complex motion dynamics over time.

Existing methods \citep{liang2024movideo, niu2024mofa, wang2024motionctrl} address this by integrating explicit motion control into frozen diffusion models.  However, these approaches predominantly adopt a Lagrangian motion formulation, where optical flow is computed relative to the initial reference frame \citep{fraser2005leonhard}. While conceptually simple, we identify two inherent structural vulnerabilities in this reference-anchored paradigm that make it highly susceptible to error accumulation. First, as the timestep increases, the displacement magnitude grows, often violating the brightness constancy assumptions inherent in optical flow estimation \citep{teed2020raft} and leading to increasingly noisy supervisory signals. Second, supervisory sparsity: as objects move or rotate, the valid correspondence area between the initial frame and the current frame decays exponentially due to occlusions, leaving the model hallucinating in unseen regions without guidance.

To mitigate these limitations, we propose a shift from Lagrangian-based to Eulerian-based Motion Guidance. Instead of tracking pixels relative to the start, we model motion as a dense field of transitions between consecutive frames. We provide a theoretical framework demonstrating how this formulation keeps optical flow estimates within the reliable, short-range regime of estimators, effectively bounding the per-step supervisory error. Moreover, this adjacent formulation allows for a parallelized batched flow computation strategy that minimizes sequential training overhead.

However, while Eulerian motion guidance provides stable local gradients, enforcing consistency only between neighboring frames can still lead to stochastic drift in regions that are newly revealed over time. In these dis-occluded areas, there is no valid correspondence to the previous frame; forcing the model to guess the appearance often results in shimmering or ghosting artifacts. To address this inherent limitation of adjacent-frame tracking, we introduce a novel Bidirectional Geometric Consistency mechanism. By computing flow fields in both forward and backward directions, we enforce a geometric cycle-consistency check, allowing us to mathematically derive an occlusion mask that dynamically filters out unreliable gradients in dis-occluded regions, ensuring that the model is supervised only by geometrically valid correspondences.

In summary, our contributions are three-fold:
\begin{itemize}
\item We formally analyze the error-accumulation vulnerabilities of Lagrangian motion guidance in video generation and propose an Eulerian alternative designed to bound per-step supervisory error.
\item We introduce a Bidirectional Geometric Consistency mechanism that leverages forward-backward cycle checks to robustly handle occlusions and prevent artifact generation in dynamic scenes.
\item We implement batched flow computation that integrates seamlessly with pre-trained diffusion backbones, achieving state-of-the-art performance in temporal consistency and visual fidelity.
\end{itemize}
\section{Related Work}
\begin{figure*}[t]
    \centering
    \includegraphics[width=0.85\linewidth]{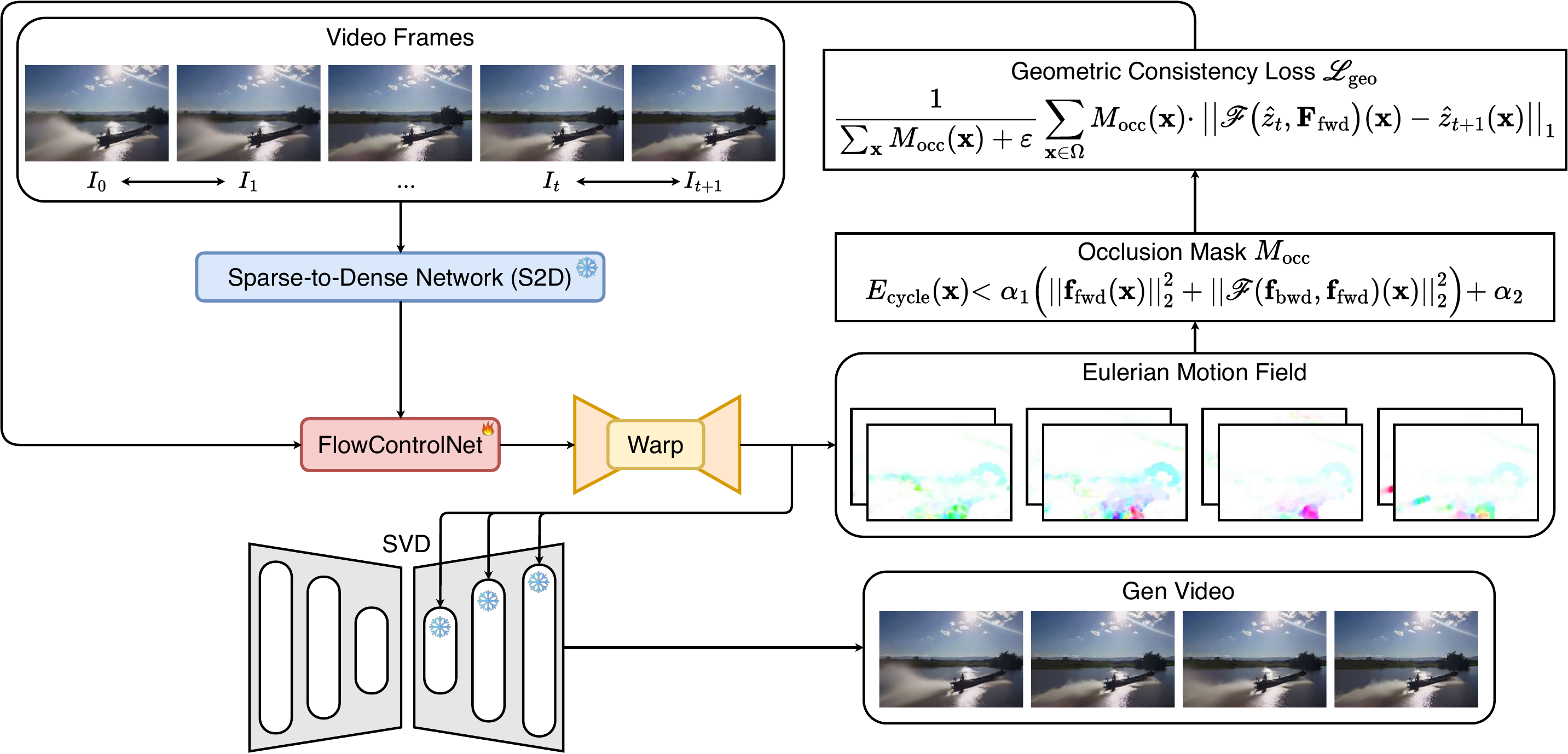}
    \caption{Eulerian Motion Guidance with Bidirectional Geometric Consistency. Given a reference image and sparse motion hints, a sparse-to-dense (S2D) module predicts an Eulerian motion field, which we then inject via Flow ControlNet by warping multi-scale reference features in a frozen SVD backbone. During training, forward and backward flows produce a cycle-based occlusion mask that gates geometric consistency loss, preventing supervision on dis-occluded regions.}
    \label{fig:overall_paradigm}
\end{figure*}

\begin{figure}[t]
    \centering
    \includegraphics[width=0.85\linewidth]{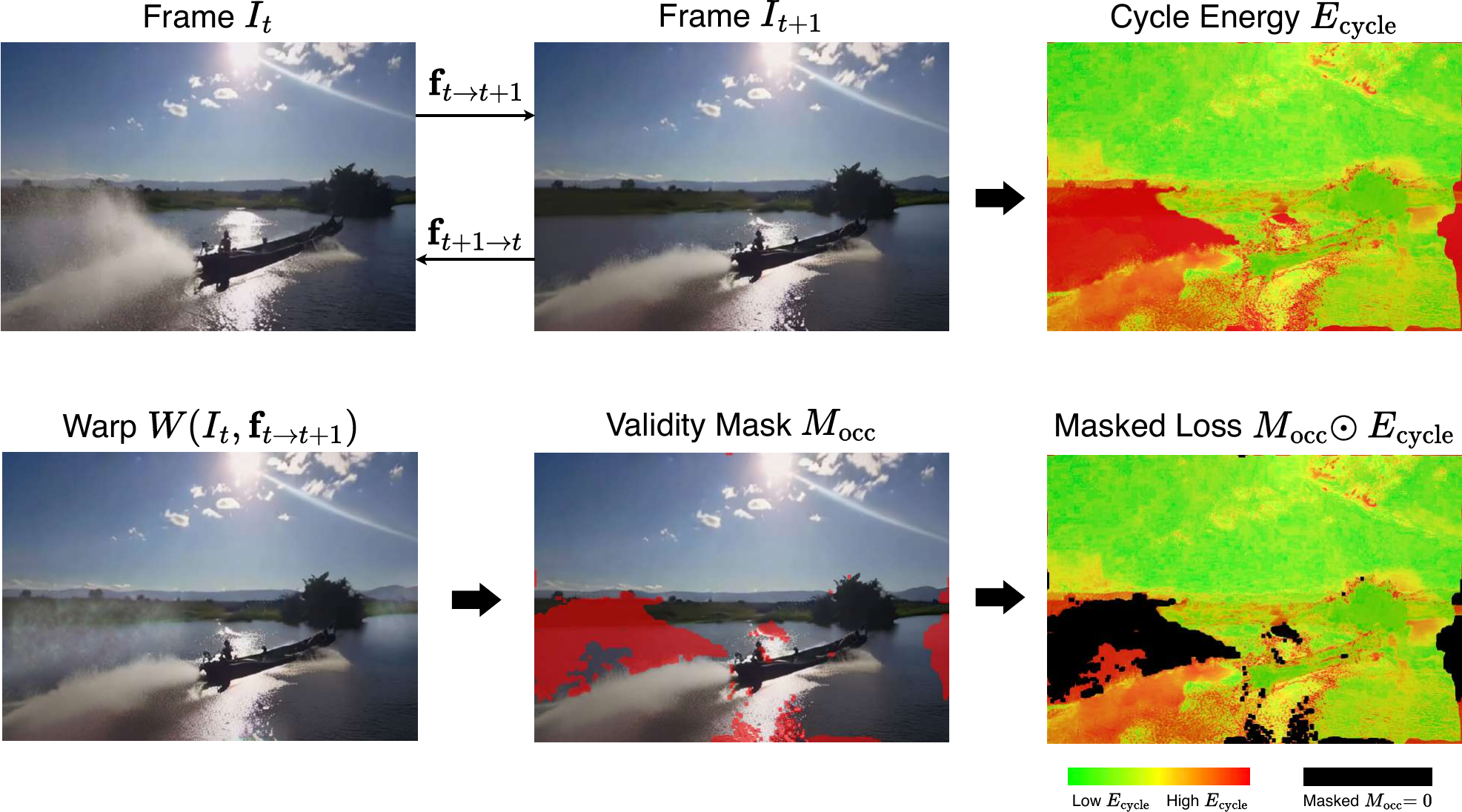}
    \caption{Occlusion masking from bidirectional cycle energy. For consecutive frames, we compute forward and backward flows, and evaluate cycle energy, for which high values indicate occlusion or flow failure. Thresholding yields a binary validity or occlusion mask, which is applied to the warp-based consistency loss so gradients are propagated only through geometrically verifiable correspondences.}
    \label{fig:occlusion_mask}
    \vspace{-15pt}
\end{figure}

\subsection{Controllable Video Generation.}
Recent large-scale video generation models, such as Veo3 \citep{wiedemer2025video} and Wan2.1 \citep{wan2025wan}, have demonstrated remarkable fidelity in text-to-video synthesis. However, fine-grained control remains a challenge. Existing controllable frameworks typically inject guidance via additional control layers or adapters. For trajectory-based control, early works like DragNUWA \citep{yin2023dragnuwa} and MotionCtrl \citep{wang2024motionctrl} utilize sparse tracking points to guide generation. More recently, ATI \citep{wang2025ati} propose a unified trajectory instruction mechanism for camera and object motion, while PoseTraj \citep{ji2025posetraj} introduce pose-aware guidance to handle object rotations better. Despite these advancements, most prior methods rely on Lagrangian particle tracking-anchoring motion to the initial frame, which we argue leads to error accumulation in long-horizon generation. Go-with-the-flow \citep{burgert2025go} recently integrates optical flow into diffusion noise for motion control. However, unlike our explicit Eulerian flux, they operate via real-time noise warping, which can be sensitive to noise distribution shifts.

\subsection{Portrait and Keypoint-Based Animation}
Animating portraits from static images requires precise identity preservation. Foundational works in consistent animation, such as MagicAnimate \citep{xu2024magicanimate}, SadTalker \citep{zhang2023sadtalker}, Cinemo \citep{ma2024cinemo}, and Hallo3 \citep{cui2025hallo3}, tackle consistent and controllable image animation using dense guidance and spatial-temporal attention mechanisms. Specialized architectures like EchoMimic \citep{chen2025echomimic} and FactorPortrait \citep{tang2025factorportrait} achieve state-of-the-art disentanglement using implicit control. While these methods achieve high fidelity through domain-specific priors, our work demonstrates that a general-purpose Eulerian flow field can achieve competitive structural consistency across broader domains.

\subsection{Geometric Consistency and Occlusion Handling}
Maintaining temporal consistency in generative video, particularly under occlusion, is a long-standing problem. Traditional approaches in video magnification, such as MagFormer \citep{gao2022magformer}, have explored combining Lagrangian and Eulerian perspectives. Furthermore, previous generative pipelines have leveraged occlusion maps, such as LFDM \cite{ni2023conditional} for latent flow diffusion and FlowVid \cite{liang2023flowvid} for forward-backward consistency in video editing. Confidence-based warping and occlusion-aware self-supervision have also been extensively explored in Video Frame Interpolation (VFI) \citep{Jeong_2024_CVPR}. 

While these methods utilize similar forward-backward checks, our Bidirectional Geometric Consistency (BGC) uniquely integrates this cycle-energy threshold directly into the diffusion training loop to gate dis-occluded gradients. By enforcing a forward-backward cycle check directly in the flow field, we achieve robust occlusion handling and texture preservation without the computational overhead of explicit 3D reconstruction seen in recent methods like Unboxed \citep{yu2025unboxed} or InsertAnywhere \citep{jin2025insertanywhere}.

\section{Preliminaries}
\label{sect:preliminaries}

\noindent\textbf{Stable Video Diffusion.} We employ \emph{Stable Video Diffusion (SVD)} as our frozen generative backbone. SVD operates in a compressed latent space derived from a pre-trained VAE encoder, $z_0 = \mathcal{E}(V)$. The generation process reverses a noise perturbation chain via a 3D U-Net denoiser $S_{\theta}(z_t, t, c)$, which predicts the velocity field $v_{\theta}$ given a noisy latent $z_t$ and a context embedding $c$, e.g., CLIP image embedding. 

\noindent\textbf{Controllable Image Animation.}
To inject motion control, we integrate a learnable motion adapter $\mathcal{M}$ into the frozen $S_{\theta}$. The adapter transforms sparse user motion hints $F^s$ into a dense flow field, which spatially aligns multi-scale reference features $f_0^{(k)}$ to the target timestep $t$ via a differentiable warping operator $\mathcal{W}$. We optimize only the adapter parameters $\theta_{\mathcal{M}}$ using a standard reconstruction objective $\mathcal{L}_{\text{control}}$, keeping the backbone frozen.

\section{Method}
In this section, we delineate our Eulerian Motion Guidance (EMG) framework for image animation. We begin by establishing the theoretical intuition behind our approach through a formal analysis of expected error propagation in optical flow fields. Subsequently, we detail the Bidirectional Geometric Consistency regularizer, formulated as an energy minimization problem. Eventually, we demonstrate the parallelized motion computation for training efficiency and our inference pipeline for image animation.

\subsection{From Lagrangian to Eulerian Motion Field}
Most image animation frameworks guide motion in a Lagrangian manner: they estimate a reference-to-target displacement field from the reference frame at $t=0$ to a future time $t$, and warp the reference latent accordingly. Concretely, given the reference image $I_0$ and its latent $z_0 = \mathcal{E}(I_0)$, they synthesize a motion-guided latent at time $t$ via
\begin{equation}
\hat{z}_t^{\mathrm{Lag}} = \mathcal{W}\!\left(z_0, \hat{\mathbf{u}}_{0\rightarrow t}\right),
\end{equation}
where $\mathcal{W}(\cdot,\cdot)$ is a differentiable warping operator (e.g., bilinear sampling), and $\hat{\mathbf{u}}_{0\rightarrow t}:\mathrm{\Omega}\rightarrow\mathbb{R}^2$ is the \emph{estimated} displacement field defined on the reference coordinates.

\begin{definition}[Lagrangian motion field]
Let $\mathrm{\Omega} \subset \mathbb{R}^{2}$ denote the spatial domain of the reference frame at $t=0$.
The Lagrangian motion field is a reference-anchored displacement function
$\mathbf{u}_{\mathrm{Lag}}:\mathrm{\Omega}\times\mathbb{R}^{+}\rightarrow\mathbb{R}^{2}$ that maps each pixel coordinate
$\mathbf{x}\in\mathrm{\Omega}$ at $t=0$ to its cumulative displacement at time $t$. The corresponding location at time $t$ is
\begin{equation}
\mathbf{x}_t = \mathbf{x} + \mathbf{u}_{\mathrm{Lag}}(\mathbf{x}, t).
\end{equation}
For discrete frames, the reference-to-$t$ displacement used in warping is exactly
$\mathbf{u}_{0\rightarrow t}(\mathbf{x}) \equiv \mathbf{u}_{\mathrm{Lag}}(\mathbf{x}, t)$.
\end{definition}

\noindent While conceptually simple, the reference-to-$t$ formulation becomes increasingly vulnerable to error as $t$ grows:
(i) the displacement magnitude $\|\mathbf{u}_{0\rightarrow t}\|$ typically increases, making estimation and sampling more sensitive to error; and
(ii) the set of reference pixels with valid correspondences,
$\mathrm{\Omega}_{\mathrm{valid}}(t)\subseteq\mathrm{\Omega}$, shrinks due to occlusion and out-of-bound mappings.
Both effects amplify uncertainty in $\hat{\mathbf{u}}_{0\rightarrow t}$ and thus degrade the warped latent $\hat{z}_t^{\mathrm{Lag}}$ over long horizons.

\begin{theorem}[Expected Error Accumulation in Lagrangian Motion Fields]
\label{theorem_1}
Let $\mathbf{u}^{*}_{0\to t}$ be the ground-truth displacement from frame $0$ to frame $t$. Under the standard assumption that the estimator $\hat{\mathbf{u}}_{0\to t} = \mathbf{u}^{*}_{0\to t} + \boldsymbol{\epsilon}_{t}$ exhibits bounded kurtosis $\kappa$ and an error variance that scales with the temporal baseline (due to growing deformation and decreasing correspondences), such that $\mathbb{E}[\boldsymbol{\epsilon}_{t}]=\mathbf{0}$ and $\mathbb{E}[\|\boldsymbol{\epsilon}_{t}\|^{2}] \ge \sigma^{2} t$ for some $\sigma>0$, the expected endpoint error is lower bounded by
\begin{equation}
\mathbb{E}\!\left[\left\|\hat{\mathbf{u}}_{0\to t}-\mathbf{u}^{*}_{0\to t}\right\|\right]
\;\ge\; \sigma \sqrt{t}.
\end{equation}
\end{theorem}
We provide the proof of Theorem \ref{theorem_1} in Appendix \ref{sec:appendix_a}. This growth implies that the supervisory target $\hat{z}^{\mathrm{Lag}}_{t}$ becomes progressively less correlated with the true geometric structure at large $t$, forcing the denoising network to learn under increasingly high-variance guidance. To mitigate this instability, we switch from a Lagrangian motion field to an Eulerian one for image animation. Instead of predicting a single long-horizon displacement $\hat{\mathbf{u}}_{0\rightarrow t}$, we model motion as an instantaneous flux that transports the current state to the next state, and compose local transports over time.

\begin{definition}[Eulerian Motion Field]
Let $\mathrm{\Omega} \subset \mathbb{R}^{2}$ denote the spatial domain. The Eulerian motion field is a velocity function
$\mathbf{u}_{\mathrm{Eul}}:\mathrm{\Omega}\times\mathbb{R}^{+}\rightarrow\mathbb{R}^{2}$ defined at fixed coordinates,
where $\mathbf{u}_{\mathrm{Eul}}(\mathbf{x},t)$ specifies the instantaneous motion (flux) observed at location $\mathbf{x}\in\Omega$ and time $t$.
In continuous time, it induces pixel trajectories $\mathbf{x}(t)$ by the ODE
\begin{equation}
\frac{d\mathbf{x}(t)}{dt}=\mathbf{u}_{\mathrm{Eul}}(\mathbf{x}(t),t),\qquad \mathbf{x}(0)=\mathbf{x}_0.
\end{equation}
For discrete frames, we represent this flux using adjacent-frame displacement (optical flow)
$\mathbf{f}_{\tau\rightarrow\tau+1}:\mathrm{\Omega}\rightarrow\mathbb{R}^2$, yielding the update
\begin{equation}
\mathbf{x}_{t+1}=\mathbf{x}_{t}+\mathbf{f}_{t\rightarrow t+1}(\mathbf{x}_{t}).
\end{equation}
\end{definition}
Let $\mathrm{\Omega}_{\mathrm{valid}}(t)\subseteq\mathrm{\Omega}$ denote the set of pixels at time $t$ that admit valid adjacent correspondences, \textit{i.e.} not occluded and not mapped out of bounds under $\mathbf{f}_{t\to t+1}$. We now formalize how Eulerian supervision bounds the per-step gradient signal.
\begin{theorem}[Uniform Bound on Eulerian Supervisory Error]
\label{theorem_2}
Let $\mathbf{f}^{*}_{t\to t+1}$ be the ground-truth adjacent-frame displacement, and let
$\hat{\mathbf{f}}_{t\to t+1} = \mathbf{f}^{*}_{t\to t+1} + \boldsymbol{\epsilon}_{t}$
be an estimator whose error satisfies $\mathbb{E}[\boldsymbol{\epsilon}_{t}]=\mathbf{0}$ and
\begin{equation}
\mathbb{E}\!\left[\frac{1}{|\mathrm{\Omega}_{\mathrm{valid}}(t)|}\int_{\mathrm{\Omega}_{\mathrm{valid}}(t)}
\left\|\boldsymbol{\epsilon}_{t}(\mathbf{x})\right\|^{2}\, d\mathbf{x}\right] \;\le\; \sigma^{2},
\qquad \forall t,
\end{equation}
for some constant $\sigma>0$ independent of $t$.
Then the expected endpoint error of Eulerian supervision is uniformly bounded:
\begin{equation}
\mathbb{E}\!\left[\frac{1}{|\mathrm{\Omega}_{\mathrm{valid}}(t)|}\int_{\mathrm{\Omega}_{\mathrm{valid}}(t)}
\left\|\hat{\mathbf{f}}_{t\to t+1}(\mathbf{x})-\mathbf{f}^{*}_{t\to t+1}(\mathbf{x})\right\|\, d\mathbf{x}\right]
\;\le\; \sigma,
\qquad \forall t.
\end{equation}
\end{theorem}
We provide the proof of Theorem \ref{theorem_2} in Appendix \ref{app:proof-thm2}. By minimizing the temporal interval to $\delta t = 1$, we ensure the displacement magnitude $\|\mathbf{f}_{t \to t+1}\|$ remains within the high-fidelity linear regime of the flow estimator, providing a stable, low-variance gradient signal throughout the entire generation window.
\subsection{Bidirectional Geometric Consistency}
\label{subsect:bgc}
While Eulerian formulation bounds the local supervisory error, it introduces a risk of generative stochastic drift. Without a global anchor $I_0$, microscopic errors can accumulate autoregressively, leading to texture sliding in dis-occluded regions where the flow field $\mathbf{f}_{t \to t+1}$ is ill-defined. To regularize this, we propose the Bidirectional Geometric Consistency mechanism, formulated as a pixel-wise energy minimization problem.

\noindent\textbf{Cycle Energy Formulation.} We define the Cycle Consistency Energy $E_{\text{cycle}}(\mathbf{x})$ at coordinate $\mathbf{x} \in \mathrm{\Omega}$ as the residual displacement magnitude after a forward-backward projection cycle:
\begin{equation}
E_{\text{cycle}}(\mathbf{x}) = \left\| \mathbf{f}_{t \to t+1}(\mathbf{x}) + \mathcal{F}\left(\mathbf{f}_{t+1 \to t}, \mathbf{f}_{t \to t+1}\right)(\mathbf{x}) \right\|_2^2,
\end{equation}
where $\mathcal{F}(\mathbf{f}_{\text{bwd}}, \mathbf{f}_{\text{fwd}})$ denotes the backward flow vector sampled at the target location mapped by the forward flow. Non-zero energy, \textit{i.e.} $E_{cycle} \gg 0$, indicates a violation of brightness constancy, serving as a robust proxy for occlusion or estimation failure.

\noindent\textbf{Dynamic Occlusion Masking.} We derive a binary validity mask $M_\text{occ} \in \{0, 1\}^{H \times W}$ by treating the cycle energy as a random variable. To account for relative error tolerance, we employ a relaxed thresholding function parameterized by the local flow magnitude:
\begin{equation}
\begin{split}
M_\text{occ}(\mathbf{x}) &= \mathbb{I} \bigg[ E_{\text{cycle}}(\mathbf{x}) < \alpha_1 \bigg( \|\mathbf{f}_{t \to t+1}(\mathbf{x})\|_2^2 \\
&\quad + \|\mathcal{W}(\mathbf{f}_{t+1 \to t}, \mathbf{f}_{t \to t+1})(\mathbf{x})\|_2^2 \bigg) + \alpha_2 \bigg],
\end{split}
\end{equation}
where $\mathbb{I}[\cdot]$ is the indicator function, and $\alpha_1, \alpha_2$ are hyperparameters governing the sensitivity of the decision boundary. This mask effectively partitions the spatial domain into the geometrically verifiable $\mathrm{\Omega}_{\text{valid}}$ and the occluded $\mathrm{\Omega}_{\text{occ}}$.

\noindent\textbf{Geometric Consistency Loss.} The final objective function integrates this geometric prior. We define the Geometric Consistency Loss $\mathcal{L}_\text{geo}$ as a spatially-weighted reconstruction objective, ensuring the generated latent $\hat{z}_{t+1}$ adheres to the Eulerian flux only within valid regions:
\begin{equation}
\label{eq:bidirectional_geometric_consistency_loss}
\mathcal{L}_\text{geo} = \frac{1}{\sum_{\mathbf{x}} M_\text{occ}(\mathbf{x}) + \epsilon} \sum_{\mathbf{x} \in \mathrm{\Omega}} M_\text{occ}(\mathbf{x}) \cdot \left\| \mathcal{W}(\hat{z}_t, \mathbf{f}_{t \rightarrow t+1})(\mathbf{x}) - \hat{z}_{t+1}(\mathbf{x}) \right\|_1.
\end{equation}
By gating gradients with $M_{\text{occ}}$, we effectively sever the backpropagation graph in occluded regions, preventing the model from optimizing against hallucinating gradients. The total training objective is $\mathcal{L}_{\text{total}} = \mathcal{L}_{\text{diff}} + \lambda_\text{geo} \mathcal{L}_\text{geo}$, where $\mathcal{L}_{\text{diff}}$ is the standard diffusion loss defined in Section \ref{sect:preliminaries}.

\subsection{Training Efficiency via Parallelized Motion Computation}
While Eulerian formulation inherently implies a sequential dependency, strictly adhering to this autoregressive structure during training would incur a prohibitive linear computational cost $O(T)$. Because the groundtruth video $\mathcal{V}$ is fully observable during training, we exploit this non-causal availability to circumvent the sequential bottleneck.

We devise a parallelized motion computation strategy that reconfigures the input latent tensor $\mathcal{V} \in \mathbb{R}^{B \times T \times C \times H \times W}$ into a batched collection of temporal dyads. Particularly, we construct the sets of forward and backward pairs as $\mathcal{P}_\text{fwd} = \{(I_t, I_{t+1})\}_{t=0}^{T-2}$ and $\mathcal{P}_\text{bwd} = \{(I_{t+1}, I_t)\}_{t=0}^{T-2}$, respectively. By treating the temporal dimension as an extension of the batch dimension, we employ the shared, frozen flow estimator $\mathrm{\Phi}(\cdot; \theta_\mathrm{\Phi})$ to regress the bidirectional flow fields simultaneously:
\begin{gather}
\mathbf{F}_\text{fwd} = \mathrm{\Phi}(\mathcal{P}_\text{fwd}) \in \mathbb{R}^{B(T-1) \times 2 \times H \times W}, \\
\mathbf{F}_\text{bwd} = \mathrm{\Phi}(\mathcal{P}_\text{bwd}) \in \mathbb{R}^{B(T-1) \times 2 \times H \times W}
\end{gather}
This operation effectively flattens the sequential bottleneck of motion computation. While the total computational work remains $O(T)$, this strategy enables $O(1)$ wall-clock training time relative to the sequence length, bounded only by the parallel hardware capacity. The resulting factor $\mathbf{F}_\text{fwd}$ is then fed into the Motion Adapter $\mathcal{M}$ to condition the denoising U-Net, while the bidirectional pair $(\mathbf{F}_\text{fwd}, \mathbf{F}_\text{bwd})$ facilitates geometric consistency.

\subsection{Inference}

During inference, we model the generation process as a conditional autoregressive sampling chain. 

\noindent\textbf{Sparse-to-Dense Warping.} Given an initial frame $I_{0}$ and a set of sparse user trajectories $\mathcal{T} = \{(\mathbf{p}_i, \mathbf{v}_i)\}_{i=1}^N$, we first interpolate $\mathcal{T}$  to obtain a sparse hint map $F_{\text{sparse}}$. Motion adapter $\mathcal{M}$ acts as a generative prior to warp these sparse hints into Eulerian field $\hat{\mathbf{F}}_{\text{fwd}}$.

\noindent\textbf{Autoregressive Generative Chain.} The video sequence $V_{1:T}$ is generated via a Markov chain factorization:
\begin{equation}
p(V_{1:T} | I_0, \mathcal{C}) = \prod_{t=1}^{T} p(I_t | I_{t-1}, \hat{\mathbf{f}}_{t-1 \to t})
\end{equation}
At each step $t$, the conditional probability $p(I_{t}|\cdot)$ is sampled via the reverse diffusion process. Specifically, the latent $z_{t-1}$ is warped by the Eulerian flow $\hat{\mathbf{f}}_{t-1 \to t}$ to serve as the structural condition for synthesizing $z_{t}$. Thanks to the bounded per-step supervisory error of our Eulerian formulation, coupled with bidirectional checks, this autoregressive process prevents the structural degradation characteristic of Lagrangian warping, enabling the generation of consistent sequences with complex non-linear dynamics.
\begin{table*}[h!]
\centering
\caption{Trajectory-based image animation on WebVid test set (1000 samples). Lower is better for LPIPS/FID/FVD/$E_{warp}$; higher is better for CLIP-Cons. \textbf{Pref.} is the user-study win-rate (\%) of ours vs each baseline (95\% CI in brackets).}
\label{tab:sota_traj}
\resizebox{0.8\linewidth}{!}{
\begin{tabular}{lccccccc}
\toprule
\textbf{Method} & \textbf{Backbone} & \textbf{LPIPS} $\downarrow$ & \textbf{FID} $\downarrow$ & \textbf{FVD} $\downarrow$ & \textbf{CLIP-Cons} $\uparrow$ & \textbf{$E_{\text{warp}}$} ($\times 10^{-3}$) $\downarrow$ & \textbf{Pref.} $\uparrow$ \\
\midrule
DragNUWA~\citep{yin2023dragnuwa}        & SVD & 0.2705 & 19.66 & 91.38 & 0.9302 & 4.12 & 89.2\,{\scriptsize [85.1, 93.3]} \\
PoseTraj~\citep{ji2025posetraj}        & SVD & 0.1704 & 12.22 & 77.69 & 0.9562 & 2.58 & 63.5\,{\scriptsize [58.2, 68.8]} \\
MOFA~\citep{niu2024mofa}               & SVD & 0.2274 & 16.82 & 86.76 & 0.9390 & 3.45 & 78.4\,{\scriptsize [73.5, 83.3]} \\
SG-I2V~\citep{namekata2024sg}          & SVD & 0.2490 & 18.24 & 89.07 & 0.9346 & 3.81 & 82.1\,{\scriptsize [77.5, 86.7]} \\
I2VControl~\citep{feng2025i2vcontrol}  & MagicVideo-V2 & 0.2132 & 14.52 & 82.23 & 0.9476 & 3.10 & 71.3\,{\scriptsize [66.1, 76.5]} \\
AnyTraj~\citep{wang2025ati}            & DiT & 0.1989 & 13.75 & 80.71 & 0.9505 & 2.76 & 68.7\,{\scriptsize [63.2, 74.2]} \\
ImageConductor~\citep{li2025image}     & AnimateDiff & 0.1847 & 12.98 & 79.20 & 0.9534 & 2.51 & 65.2\,{\scriptsize [59.8, 70.6]} \\
\midrule
\textbf{Ours}                          & SVD & \textbf{0.1562} & \textbf{11.45} & \textbf{76.18} & \textbf{0.9591} & \textbf{1.84} & \textbf{92.3}\,{\scriptsize [\textbf{87.1}, \textbf{97.5}]} \\
\bottomrule
\end{tabular}}
\end{table*}

\begin{table}[t]
\centering
\caption{Keypoint-based portrait animation comparison. Lower is better for $E_{warp}$; Higher is better for CPBD/ArcFace/CLIP-Cons.}
\label{tab:sota_kpt}
\resizebox{\linewidth}{!}{
\begin{tabular}{lccccc}
\toprule
\textbf{Method} & \textbf{CPBD} $\uparrow$ & \textbf{ArcFace} $\uparrow$ & \textbf{CLIP-Cons} $\uparrow$ & \textbf{$E_{\text{warp}}$} ($\times 10^{-3}$) $\downarrow$ & \textbf{Pref.} $\uparrow$ \\
\midrule
SadTalker~\citep{zhang2023sadtalker} & 0.3218 & 0.9188 & 0.9156 & 3.88 & 88.5\,{\scriptsize [84.1, 92.9]} \\
MagicAnimate~\citep{xu2024magicanimate} & 0.3852 & 0.9241 & 0.9288 & 2.95 & 83.7\,{\scriptsize [78.5, 88.4]} \\
Cinemo~\citep{ma2024cinemo} & 0.3986 & 0.9315 & 0.9342 & 2.61 & 77.1\,{\scriptsize [71.8, 82.3]} \\
MOFA~\citep{niu2024mofa}             & 0.4075 & 0.9293 & 0.9390 & 2.45 & 79.2\,{\scriptsize [74.3, 84.1]} \\
Hallo~\citep{xu2024hallo}            & 0.3912 & 0.9331 & 0.9364 & 2.38 & 75.4\,{\scriptsize [70.1, 80.7]} \\
Hallo3~\citep{cui2025hallo3}         & 0.4045 & 0.9389 & 0.9409 & 2.21 & 72.8\,{\scriptsize [67.2, 78.4]} \\
FaceShot~\citep{gao2025faceshot}     & 0.4178 & 0.9439 & 0.9455 & 2.15 & 68.1\,{\scriptsize [62.4, 73.8]} \\
SVDP~\citep{yin2025stable}           & 0.4310 & 0.9485 & 0.9486 & 2.08 & 64.3\,{\scriptsize [58.5, 70.1]} \\
FactorPortrait~\citep{tang2025factorportrait} & 0.4343 & 0.9491 & 0.9459 & 2.10 & 61.7\,{\scriptsize [55.9, 67.5]} \\
\midrule
\textbf{Ours}                        & \textbf{0.4576} & \textbf{0.9558} & \textbf{0.9591} & \textbf{1.52} & \textbf{94.4}\,{\scriptsize [89.0, 99.7]} \\
\bottomrule
\end{tabular}}
\end{table}

\begin{figure*}[t]
    \centering
    \includegraphics[width=0.8\linewidth]{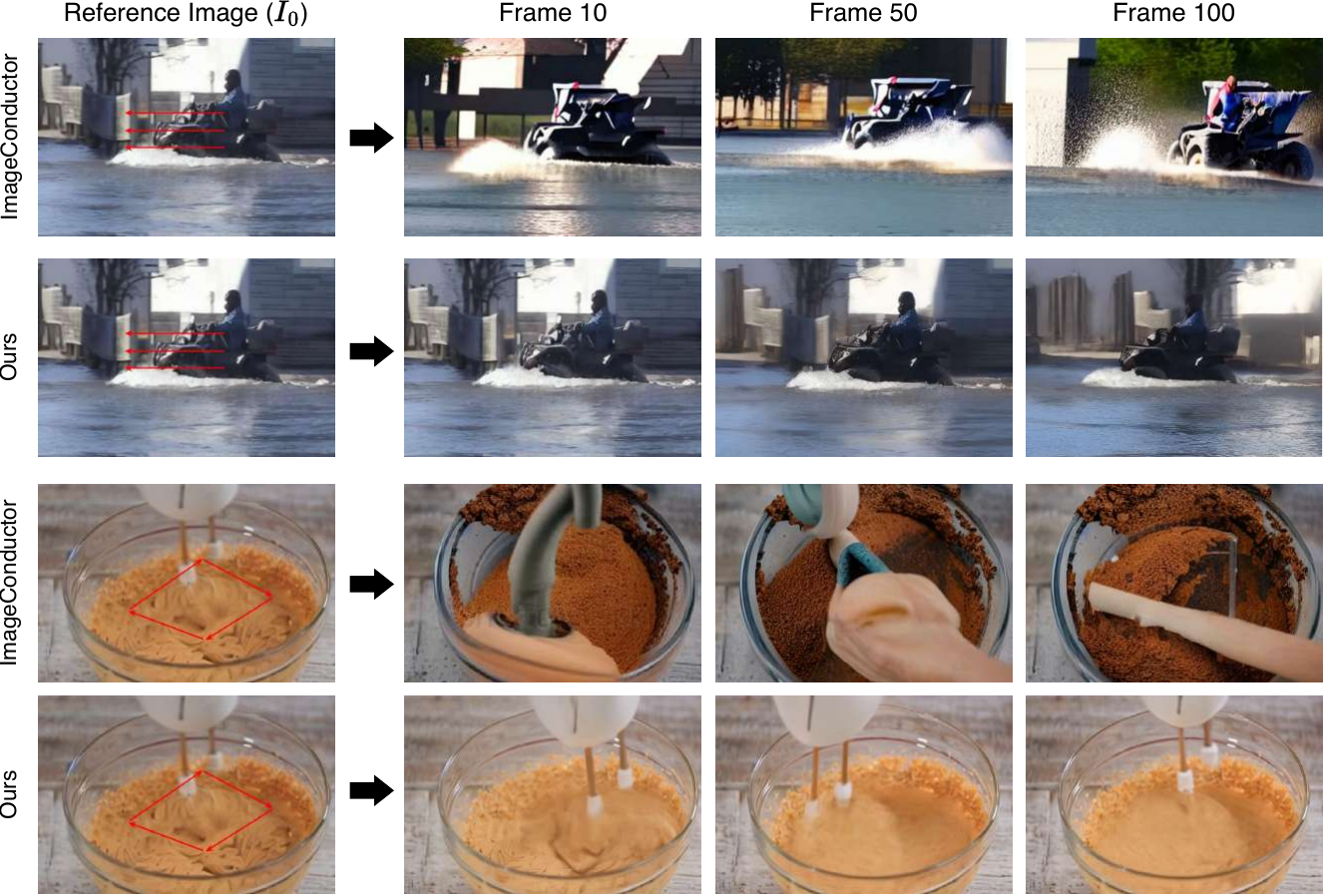}
    \caption{Qualitative comparison with ImageConductor \citep{li2025image} on long-horizon generation ($T=100$). Red arrows in $I_0$ indicate the user-defined motion trajectories. Top: ImageConductor suffers from severe structural degradation and background drift, whereas our method preserves the rider's geometry. Bottom: ImageConductor hallucinates the mixing objects, e.g. hand or sponge, due to lack of valid correspondence. In contrast, our framework maintains texture fidelity throughout the complex circular motion.}
\label{fig:qualitative_comparison_trajectory}
\vspace{-5pt}
\end{figure*}

\begin{figure*}[t]
    \centering
    \includegraphics[width=0.85\linewidth]{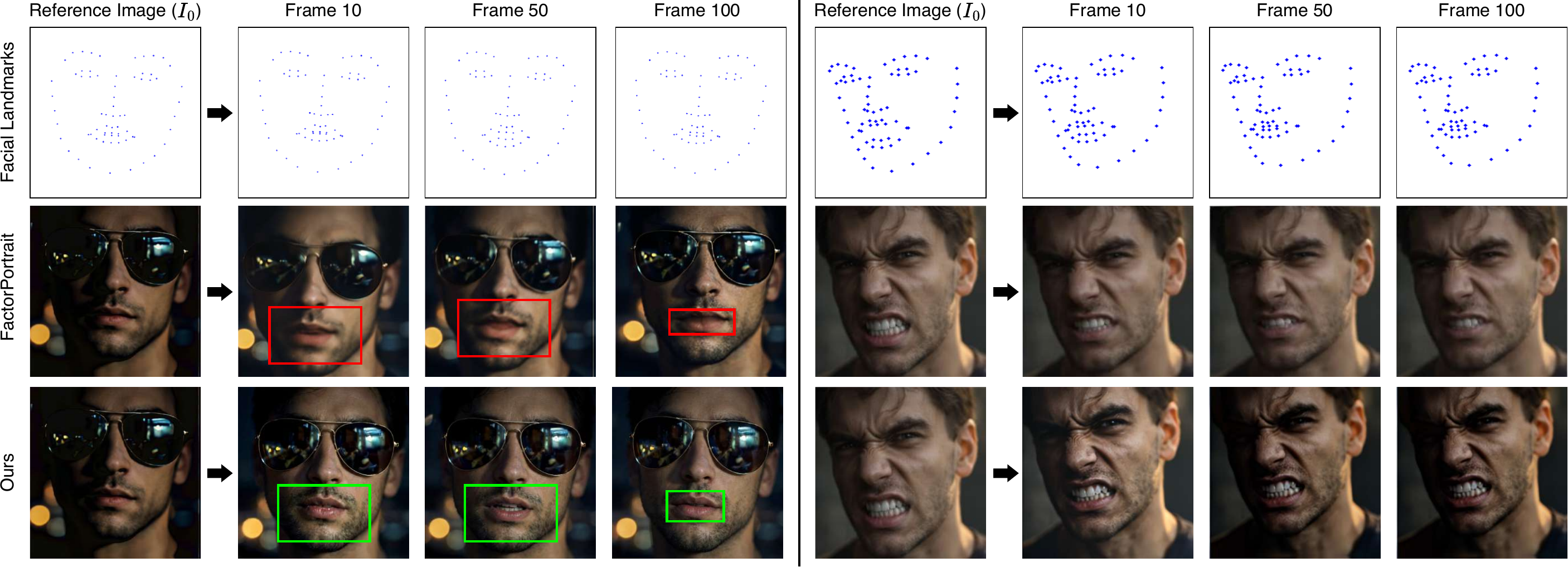}
    \caption{Qualitative comparison on keypoint-based animation. We compare our Eulerian Motion Guidance against the state-of-the-art baseline FactorPortrait \citep{tang2025factorportrait}. Left: Our method preserves high-frequency details (e.g., mouth region) where the baseline suffers from texture blurring (highlighted). This validates our claim that Eulerian flux prevents the structural degradation typical of reference-anchored warping. Right: Our method maintains identity and geometric stability even under intense expression changes.}
    \label{fig:qualitative_comparison_landmarks}
\end{figure*}

\begin{table}[t]
  \centering
  \caption{\textbf{Ablation Study.} We evaluate the contribution of Eulerian Motion Guidance and Bidirectional Geometric Consistency (BGC) on the WebVid test set.
  Lagrangian denotes a baseline trained with reference-to-target flow.
  Forward-only uses a naive forward flow magnitude mask.
  \textbf{Pref.} reports the \textit{variant win-rate} (\%) against our full model (Eulerian + BGC) in a pairwise user study (lower is better).}
  \label{tab:ablation}
  \resizebox{\linewidth}{!}{
  \renewcommand{\arraystretch}{1.1}
  \begin{tabular}{l l
                  S[table-format=1.3]
                  S[table-format=2.2]
                  S[table-format=1.3]
                  l}
    \toprule
    \textbf{Model Variant} & \textbf{Consistency Strategy} & {LPIPS $\downarrow$} & {FVD $\downarrow$} & {CLIP-Cons $\uparrow$} & \textbf{Pref.} \\
    \midrule
    Lagrangian Baseline & None & 0.224 & 85.12 & 0.935 & 12.5\% \\
    Eulerian     & None & 0.178 & 79.45 & 0.951 & 28.3\% \\
    Eulerian   & Forward-only Mask & 0.165 & 77.80 & 0.955 & 41.2\% \\
    \midrule
    Eulerian (Ours)     & \textbf{Bidirectional (BGC)} &
    {\bfseries 0.156} & {\bfseries 76.18} & {\bfseries 0.959} & -- \\
    \bottomrule
  \end{tabular}}
\end{table}

\begin{figure*}[h!]
  \centering
  \includegraphics[width=0.85\linewidth]{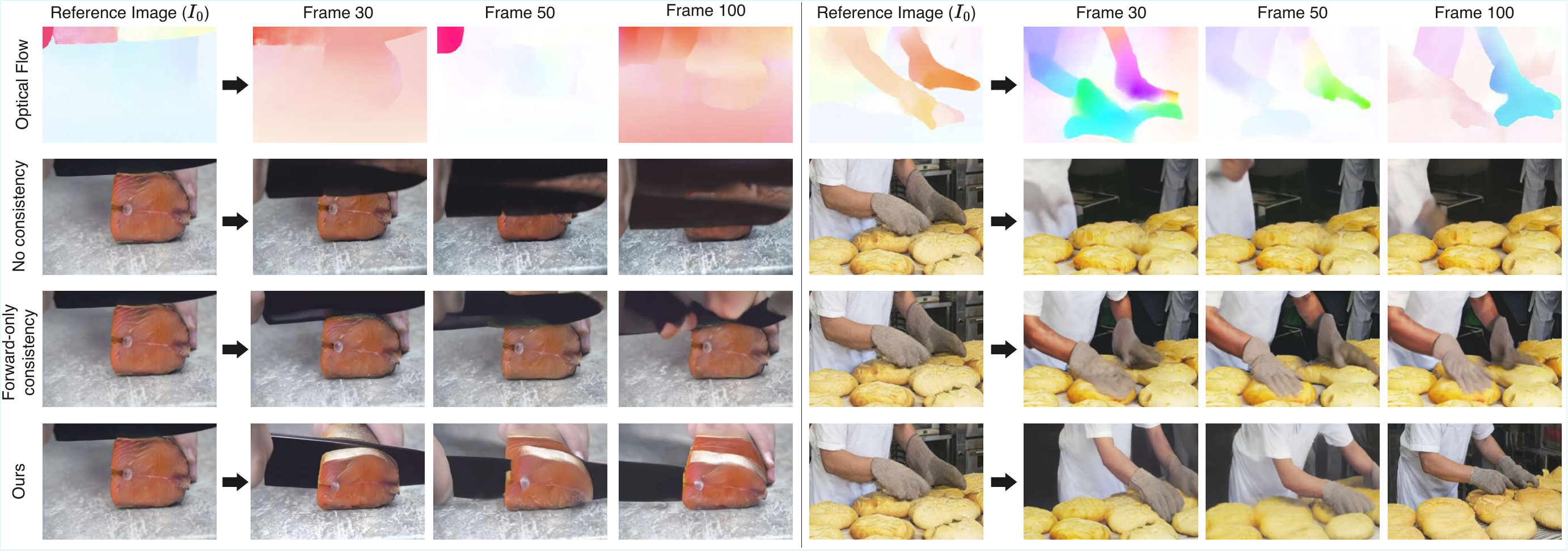}
\caption{Qualitative ablation of geometric consistency.
  We compare training with (a) no consistency enforcement, (b) forward-only masking, and (c) our bidirectional cycle-based masking (BGC),
  under the same reference image and control signal, shown at frames $t{=}30, 50, 100$.
  Without reliable masking, dis-occluded regions produce ghosting artifacts where background textures adhere to moving foreground objects.
  Forward-only masking partially mitigates drift but fails under complex occlusions.
  BGC identifies unreliable correspondences via a forward--backward cycle check and suppresses incorrect warping supervision,
  preserving structure and appearance over long horizons.}
  \label{fig:ablation_study_consistency}
\vspace{-5pt}
\end{figure*}

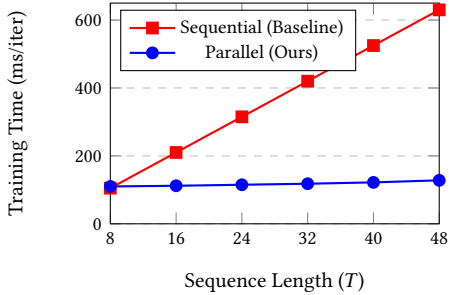
\begin{figure}[h!]
\centering
\begin{tikzpicture}
\begin{axis}[
    width=0.7\linewidth,
    height=4.5cm,
    xlabel={Sequence Length ($T$)},
    ylabel={Training Time (ms/iter)},
    xmin=8, xmax=48,
    ymin=0, ymax=650,
    xtick={8, 16, 24, 32, 40, 48},
    legend pos=north west,
    ymajorgrids=true,
    grid style=dashed,
    legend style={nodes={scale=0.8, transform shape}},
    label style={font=\small},
    tick label style={font=\footnotesize}
]
\addplot[
    color=red,
    mark=square*,
    thick
    ]
    coordinates {
    (8,105)(16,210)(24,315)(32,420)(40,525)(48,630)
    };
    \addlegendentry{Sequential (Baseline)}

\addplot[
    color=blue,
    mark=*,
    thick
    ]
    coordinates {
    (8,110)(16,112)(24,115)(32,118)(40,122)(48,128)
    };
    \addlegendentry{Parallel (Ours)}
\end{axis}
\end{tikzpicture}
\caption{\textbf{Training Efficiency Analysis.} comparison of per-iteration training time between standard sequential flow estimation (Red) and our parallelized strategy (Blue). Our method achieves $\mathcal{O}(1)$ complexity relative to sequence length, yielding a ${\sim}3\times$ speedup at $T=24$.}
\label{fig:efficiency_plot}
  \vspace{-15pt}
\end{figure}

\section{Experiments}
\subsection{Experimental Setup}
\noindent\textbf{Implementation Details.} We adopt Stable Video Diffusion (SVD) \citep{blattmann2023stable} as a frozen generative backbone and train only the FlowControlNet on WebVid-10M using AdamW \citep{loshchilov2017decoupled} with a learning rate of $2 \times 10^{-5}$ at a resolution of $256 \times 256$. For our shared flow estimator $\mathrm{\Phi}$, we utilize RAFT \citep{teed2020raft} pre-trained on FlyingThings3D, without further fine-tuning, operating at $256 \times 256$ resolution. The geometric consistency weights are empirically set to $\alpha_{1} = 0.01$ and $\alpha_{2} = 0.5$. All experiments are conducted on 4 NVIDIA H200 GPUs. At inference time, we perform image animation by generating sequences of $T = 100$ frames. 

\noindent\textbf{Evaluation Metrics.} We evaluate our framework on both trajectory-based and keypoint-based image animation. For trajectory-based image animation, we randomly sample 1000 instances from the WebVid test set and report LPIPS \citep{zhang2018perceptual}, FID \citep{heusel2017gans}, and FVD \citep{yu2022generating}. Specifically, we compute FID and FVD by matching the distribution of our generated 100-frame sequences—decoded at a uniform 8 fps—directly against the ground-truth WebVid clips to ensure spatial and temporal distributions are strictly aligned. To assess temporal consistency, we compute the cosine similarity between CLIP embeddings \citep{radford2021learning} of consecutive generated frames. Because CLIP-Cons can inadvertently penalize intended large motions, we also report the Warping Error ($E_{\text{warp}}$) \citep{lai2018learning}, which measures the pixel-level photometric error between consecutive frames after alignment via optical flow, thereby decoupling temporal consistency from motion magnitude. For keypoint-based image animation, we evaluate 40 generated videos, each of which contains 196 frames, then report Cumulative Probability Blur Detection (CPBD) \citep{narvekar2011no} to measure sharpness, as well as identity preservation via ArcFace \citep{deng2019arcface} between the source images and the generated frames. In addition to quantitative results, we also conduct user studies to compare perceptual quality against competing baselines.

\subsection{Comparison with State-of-the-art Methods}

In this section, we compare our results with the state-of-the-art methods tailored to each control setting. For all baselines, we follow the official implementations when available and keep the evaluation protocol identical across methods. Quantitative results are reported in Table \ref{tab:sota_traj} and Table \ref{tab:sota_kpt}. Qualitative results are provided in Figure \ref{fig:qualitative_comparison_trajectory} and Figure \ref{fig:qualitative_comparison_landmarks}.

\noindent\textbf{Trajectory-based Image Animation.} We compare our method against recent trajectory-based video generation and image-to-video control methods, including 1) trajectory-guided diffusion fine-tuning approaches that directly inject sparse trajectories into the generation process, i.e. DragNUWA \citep{yin2023dragnuwa} and PoseTraj \citep{ji2025posetraj}; and 2) trajectory control methods built on top of the pretrained image-to-video backbones, i.e. MOFA \citep{niu2024mofa}, SG-I2V \citep{namekata2024sg}, I2VControl \citep{feng2025i2vcontrol}, AnyTraj \citep{wang2025ati}, and ImageConductor \citep{li2025image}. Following prior works, we evaluate both object dragging and camera motion scenarios. For methods that require trajectory formats, e.g. bounding box trajectories or sparse point sets, we convert our user-specified trajectories into the required representation while preserving the same spatial path and temporal schedule. As shown in Table \ref{tab:sota_traj}, we achieve substantially better perceptual scores, achieving the lowest LPIPS, FID, and FVD scores. Notably, we outperform the strongest baseline, i.e. ImageConductor \citep{li2025image}, by a significant margin of 1.53 in FID and 3.02 in FVD, validating the efficacy of our parallelized Eulerian formulation in generating temporally coherent video distributions. Crucially, addressing recent concerns regarding the limitations of CLIP-Cons in measuring alignment under large motions, our framework achieves the lowest Warping Error ($E_{warp} = 1.84 \times 10^{-3}$). This confirms that our Eulerian formulation and Bidirectional Geometric Consistency directly translate to superior pixel-level temporal stability, minimizing texture flickering even during rapid object translations.

In addition to quantitative comparison, qualitative comparison also reveals distinct advantages in long-horizon consistency, where Lagrangian-based methods typically struggle. In Figure \ref{fig:qualitative_comparison_mofa}, as temporal distance increases, MOFA exhibits severe textual drift and identity degradation. For instance, the structural identity of the red bear (row 1) and the boat's wake (row 2) collapses by frame 100 due to error accumulation in reference-anchored flow. In contrast, our Eulerian Motion Guidance maintains sharp structural details and clear object boundaries throughout the sequence. Moreover, in the jet-ski example (top row) of Figure \ref{fig:qualitative_comparison_trajectory}, ImageConductor fails to maintain the rider's geometry, leading to the rider vanishing into the background by frame 50. Our method preserves the rider's silhouette against the fast-moving water. Furthermore, in the mixing bowl scene (bottom row), ImageConductor suffers from hallucination artifacts, generating unrealistic objects in the batter due to occlusion ambiguities. By leveraging Bidirectional Geometric Consistency to mask unreliable regions, our model generates plausible fluid dynamics without such hallucinations.

\noindent\textbf{Keypoint-based Image Animation.} We compare our method with strong portrait or talking-head baselines that specialize in keypoint-based image animation, including classical audio-to-motion pipelines and recent diffusion-based portrait animators, i.e. SadTalker \citep{zhang2023sadtalker}, MagicAnimate \citep{xu2024magicanimate}, Cinemo \citep{ma2024cinemo}, StyleHeat \citep{yin2022styleheat}, Hallo \citep{xu2024hallo}, Hallo3 \citep{cui2025hallo3}, FaceShot \citep{gao2025faceshot}, and recent video-driven portrait methods, including Stable Video-Driven Portraits (SVDP) \citep{yin2025stable} and FactorPortrait \citep{tang2025factorportrait}. Table \ref{tab:sota_kpt} demonstrates that our method outperforms specialized portrait animators. We achieve the highest CPBD score, indicating superior frame sharpness, and the highest ArcFace score, reflecting better identity preservation than recent methods like Hallo3 \citep{cui2025hallo3}, Cinemo \citep{ma2024cinemo}, and FactorPortrait \citep{tang2025factorportrait}.

Furthermore, qualitative comparison in Figure \ref{fig:qualitative_comparison_landmarks} with FactorPortrait reveals our method's ability to retain fine facial details. In the sunglasses' subject (left), the baseline struggles with the dis-occluded mouth region, resulting in blurred teeth and undefined lip boundaries (highlighted in red boxes). Our method utilizes the Eulerian flux to propagate texture information from adjacent frames, yielding sharp, high-fidelity mouth movements (highlighted in green boxes). In the angry expression subject (right), our method maintains geometric stability under large deformations. While the baseline introduces temporal flicker and boundary artifacts around the jawline during intense expressions, our Bidirectional Geometric Consistency effectively filters inconsistent warping gradients, ensuring subject's identity remains stable during facial contortions.

\noindent\textbf{User Studies.} In addition to objective metrics, we conduct a user study to evaluate perceived visual quality, temporal stability, and control faithfulness. Participants are shown randomized pairwise comparisons between our method and each baseline under the same input and control signals. We ask them to select the preferred output, and report win rates and confidence intervals in Table \ref{tab:sota_traj} and \ref{tab:sota_kpt}. We can observe that the volunteers prefer the proposed framework obtains the highest preferred rates across all approaches.

\vspace{-5pt}
\subsection{Ablation Study}
To validate the effectiveness of our proposed contributions, we conduct ablation studies on the WebVid test set. We analyze three core aspects: training efficiency, motion guidance formulation, and geometric consistency.

\noindent\textbf{Training Efficiency.} 
A key advantage of our Eulerian formulation is the ability to parallelize motion estimation (Sec.~4.3). We compare the training iteration time of our method against a standard sequential implementation where flow is estimated autoregressively. Figure~\ref{fig:efficiency_plot} shows that the sequential baseline scales linearly $\mathcal{O}(T)$, creating a computational bottleneck for long sequences. In contrast, our parallelized strategy maintains nearly constant time complexity $\mathcal{O}(1)$ (relative to temporal hardware capacity). For a sequence length of $T=24$, our method reduces iteration time from 315ms to 115ms, achieving a $2.7\times$ speedup.

\noindent\textbf{Effectiveness of Eulerian Motion Guidance.} 
We train a Lagrangian variant that predicts displacements relative to the initial frame $I_0$ while keeping the SVD backbone and all other settings identical. As shown in Table~\ref{tab:ablation}, the Lagrangian baseline degrades significantly (LPIPS rises to 0.224, FVD to 85.12, CLIP-Cons drops to 0.935), confirming that reference-anchored flow suffers from error accumulation and vanishing valid correspondences over long horizons. Our Eulerian formulation bounds per-step error (Theorem~1 and~2), preserving sharp textures and structural fidelity throughout 100-frame sequences.

\noindent\textbf{Impact of Bidirectional Geometric Consistency.} 
The Eulerian formulation is prone to drift in dis-occluded regions. We compare three strategies on the Eulerian baseline: (1)~no consistency check (full-pixel supervision), (2)~forward-only magnitude masking, and (3)~our BGC using forward-backward cycle energy $E_{\text{cycle}}$. Removing consistency increases LPIPS to 0.178. The forward-only heuristic helps modestly but fails under rotation and complex occlusions. BGC achieves the best scores (LPIPS 0.156, FVD 76.18, CLIP-Cons 0.959) by gating unreliable gradients. Figure~\ref{fig:ablation_study_consistency} shows that BGC eliminates ghosting and boundary artifacts, cleanly separating foreground from newly revealed background. More comprehensive analysis can be found in Appendix \ref{app:sensitivity_analysis}.
\section{Conclusion}
In this paper, we identify a fundamental theoretical limitation in existing image animation frameworks, i.e. the unbounded error growth inherent in reference-anchored Lagrangian motion guidance. To address this, we propose Eulerian Motion Guidance, a novel paradigm that formulates animation as a sequence of adjacent, bounded-error flow transitions. We further introduce the Bidirectional Geometric Consistency mechanism, which utilizes a forward-backward cycle energy check to mathematically detect and mask unreliable gradients in dis-occluded regions. Extensive experiments on trajectory-based and keypoint-based tasks demonstrate that our approach significantly outperforms recent state-of-the-art methods. Our method not only generates sharper textures over long horizons but also eliminates the artifacts common in complex dynamic scenes.

\bibliographystyle{ACM-Reference-Format}
\bibliography{main}
\clearpage

\appendix
\section{Proof of Theorem 1}
\label{sec:appendix_a}

\begin{customthm}{1.}[Expected Error Accumulation in Lagrangian Motion Fields]
Let $\mathbf{u}^{*}_{0\to t}$ be the ground-truth displacement from frame $0$ to frame $t$. Under the standard assumption that the estimator $\hat{\mathbf{u}}_{0\to t} = \mathbf{u}^{*}_{0\to t} + \boldsymbol{\epsilon}_{t}$ exhibits an error variance that scales with the temporal baseline (due to growing deformation and decreasing correspondences), such that $\mathbb{E}[\boldsymbol{\epsilon}_{t}]=\mathbf{0}$ and $\mathbb{E}[\|\boldsymbol{\epsilon}_{t}\|^{2}] \ge \sigma^{2} t$ for some $\sigma>0$, the expected endpoint error is lower bounded by
\begin{equation}
\mathbb{E}\!\left[\left\|\hat{\mathbf{u}}_{0\to t}-\mathbf{u}^{*}_{0\to t}\right\|\right]
\;\ge\; \sigma \sqrt{t}.
\end{equation}
\end{customthm}

\begin{proof}
Let $Z = ||\epsilon_{t}||$ denote the error magnitude. By our variance-scaling assumption, the second moment satisfies $\mathbb{E}[Z^2] \ge \sigma^2 t$.
While a large second moment implies a large Mean Squared Error (MSE), it does not strictly imply a large expected L1 error ($\mathbb{E}[Z]$) without constraints on the distribution's heavy-tailedness. We utilize the Paley-Zygmund inequality to strictly lower bound the first moment.

Let $X = Z^2$. Using the Paley-Zygmund inequality with $\theta = 0.5$, we have:
\begin{equation}
    \mathbb{P}(X > 0.5 \mathbb{E}[X]) \ge \frac{(1-0.5)^2 \mathbb{E}[X]^2}{\mathbb{E}[X^2]}.
\end{equation}
Using the bounded kurtosis assumption $\mathbb{E}[X^2] = \mathbb{E}[Z^4] \le \kappa (\mathbb{E}[Z^2])^2 = \kappa \mathbb{E}[X]^2$, we obtain:
\begin{equation}
    \mathbb{P}(Z^2 > 0.5 \mathbb{E}[Z^2]) \ge \frac{0.25}{\kappa} = \frac{1}{4\kappa}.
\end{equation}
This inequality states that with probability at least $\frac{1}{4\kappa}$, the squared error is significant. We can now lower bound the expected error:
\begin{equation}
\begin{aligned}
    \mathbb{E}[Z] &\ge \mathbb{E}[Z \cdot \mathbb{I}(Z^2 > 0.5\mathbb{E}[Z^2])] \\
    &\ge \sqrt{0.5\mathbb{E}[Z^2]} \cdot \mathbb{P}(Z^2 > 0.5\mathbb{E}[Z^2]) \\
    &\ge \frac{1}{\sqrt{2}} \sigma\sqrt{t} \cdot \frac{1}{4\kappa} \\
    &= \frac{1}{4\sqrt{2}\kappa}\sigma\sqrt{t}.
\end{aligned}
\end{equation}
As $t\rightarrow\infty$, the term $\sigma\sqrt{t}\rightarrow\infty$. Thus, if the error variance scales with time, the expected endpoint error of the Lagrangian estimator diverges at a rate of $\mathrm{\Omega}(\sqrt{t})$, explaining the empirical drift observed in long-horizon generation.
\end{proof}

\section{Proof of Theorem 2}
\label{app:proof-thm2}
\begin{customthm}{2.}[Uniform Bound on Eulerian Supervisory Error]
Let $\mathbf{f}^{*}_{t\to t+1}$ be the ground-truth adjacent-frame displacement, and let
$\hat{\mathbf{f}}_{t\to t+1} = \mathbf{f}^{*}_{t\to t+1} + \boldsymbol{\epsilon}_{t}$
be an estimator whose error satisfies $\mathbb{E}[\boldsymbol{\epsilon}_{t}]=\mathbf{0}$ and
\begin{equation}
\mathbb{E}\!\left[\frac{1}{|\mathrm{\Omega}_{\mathrm{valid}}(t)|}\int_{\mathrm{\Omega}_{\mathrm{valid}}(t)}
\left\|\boldsymbol{\epsilon}_{t}(\mathbf{x})\right\|^{2}\, d\mathbf{x}\right] \;\le\; \sigma^{2},
\qquad \forall t,
\end{equation}
for some constant $\sigma>0$ independent of $t$.
Then the expected endpoint error of Eulerian supervision is uniformly bounded:
\begin{equation}
\mathbb{E}\!\left[\frac{1}{|\mathrm{\Omega}_{\mathrm{valid}}(t)|}\int_{\mathrm{\Omega}_{\mathrm{valid}}(t)}
\left\|\hat{\mathbf{f}}_{t\to t+1}(\mathbf{x})-\mathbf{f}^{*}_{t\to t+1}(\mathbf{x})\right\|\, d\mathbf{x}\right]
\;\le\; \sigma,
\qquad \forall t.
\end{equation}
\end{customthm}
\begin{proof}
Define the pointwise flow error field
\[
\varepsilon_t(x) := \hat f_{t\to t+1}(x) - f^{*}_{t\to t+1}(x),
\quad x\in\mathrm{\Omega}_{\mathrm{valid}}(t).
\]
Let
\[
A_t := \left\langle \|\varepsilon_t(\cdot)\| \right\rangle_{\mathrm{\Omega}_{\mathrm{valid}}(t)}
= \frac{1}{|\mathrm{\Omega}_{\mathrm{valid}}(t)|}\int_{\mathrm{\Omega}_{\mathrm{valid}}(t)} \|\varepsilon_t(x)\|\,dx.
\]
By Cauchy--Schwarz on $\mathrm{\Omega}_{\mathrm{valid}}(t)$,
\[
\left(\int_{\mathrm{\Omega}_{\mathrm{valid}}(t)} \|\varepsilon_t(x)\|\,dx\right)^2
\le
|\mathrm{\Omega}_{\mathrm{valid}}(t)|\int_{\mathrm{\Omega}_{\mathrm{valid}}(t)} \|\varepsilon_t(x)\|^2\,dx.
\]
Dividing both sides by $|\mathrm{\Omega}_{\mathrm{valid}}(t)|^2$ yields the deterministic inequality
\[
A_t^2 \;\le\;
\left\langle \|\varepsilon_t(\cdot)\|^2 \right\rangle_{\mathrm{\Omega}_{\mathrm{valid}}(t)}.
\]
Taking expectation and using the assumption of Theorem 2 gives
\[
\mathbb{E}[A_t^2]
\le
\mathbb{E}\!\left[\left\langle \|\varepsilon_t(\cdot)\|^2 \right\rangle_{\mathrm{\Omega}_{\mathrm{valid}}(t)}\right]
\le \sigma^2.
\]
Finally, since $A_t\ge 0$, Jensen (equivalently $\mathbb{E}[A_t]\le \sqrt{\mathbb{E}[A_t^2]}$)
implies
\[
\mathbb{E}[A_t] \;\le\; \sqrt{\mathbb{E}[A_t^2]} \;\le\; \sigma,
\]
which is exactly the claimed uniform bound for all $t$.
\end{proof}

\section{More Visual Results}
We provide extended illustrations to further examine the behavior of our method under diverse motion conditions in Figure \ref{fig:appendix_more_visual_results_optical_flow} and \ref{fig:appendix_more_visual_results_landmarks}.
\begin{figure}[h!]
  \centering
  % Replace with your actual file path:
  \includegraphics[width=0.84\linewidth]{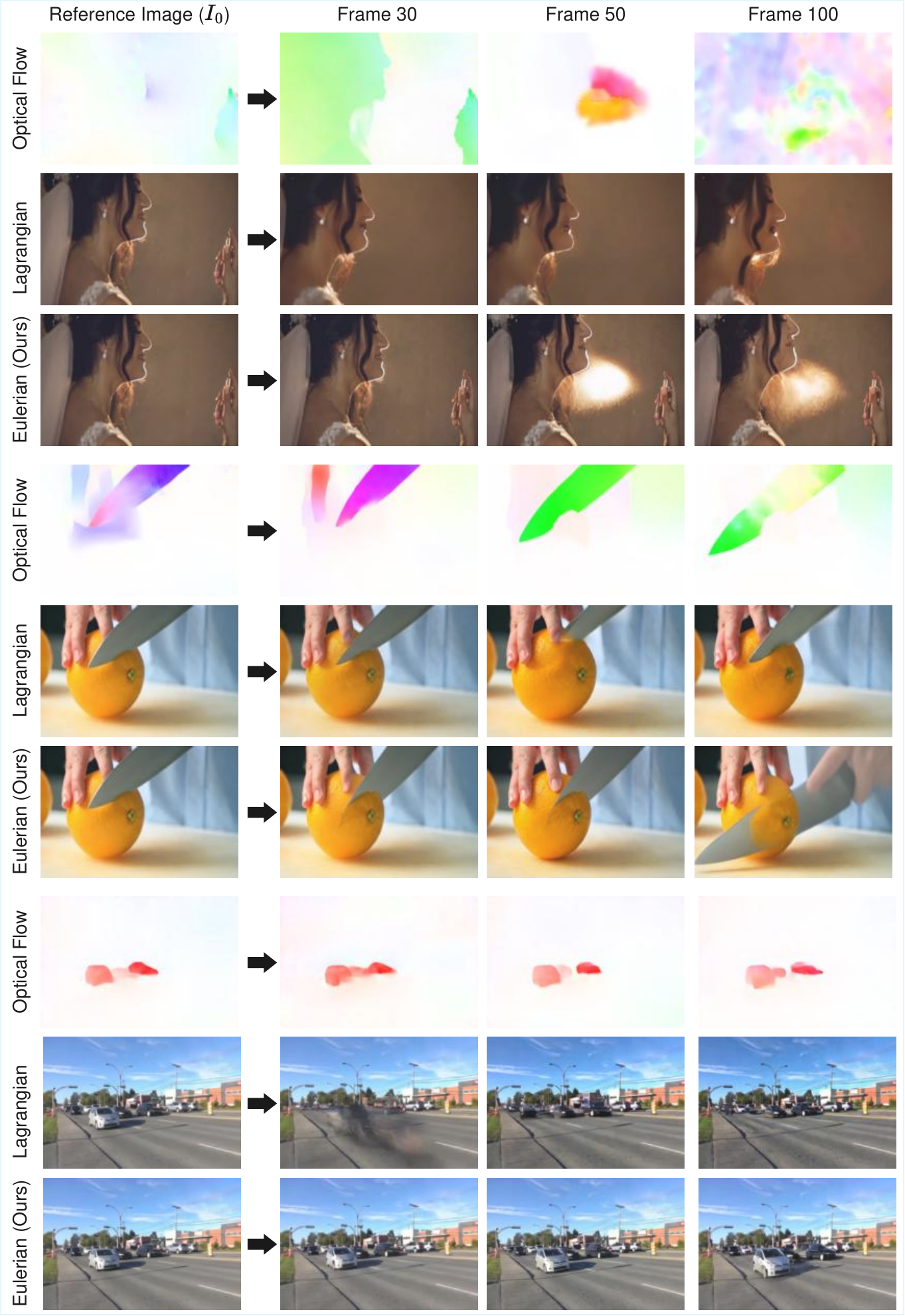}
  \caption{Robustness to Large Displacement. We compare Eulerian Motion Guidance against the Lagrangian baseline on open-domain video generation. Row 2 visualizes the optical flow magnitude, illustrating high-motion regions. In the Lagrangian baseline (Row 3), the structure of the subject degrades significantly by Frame 100 due to drift. In contrast, our Eulerian approach (Row 4) maintains sharp boundaries and coherent textures throughout the generation window.}
  \label{fig:appendix_more_visual_results_optical_flow}
\end{figure}
\begin{figure}[h!]
  \centering
  % Replace with your actual file path:
  \includegraphics[width=0.84\linewidth]{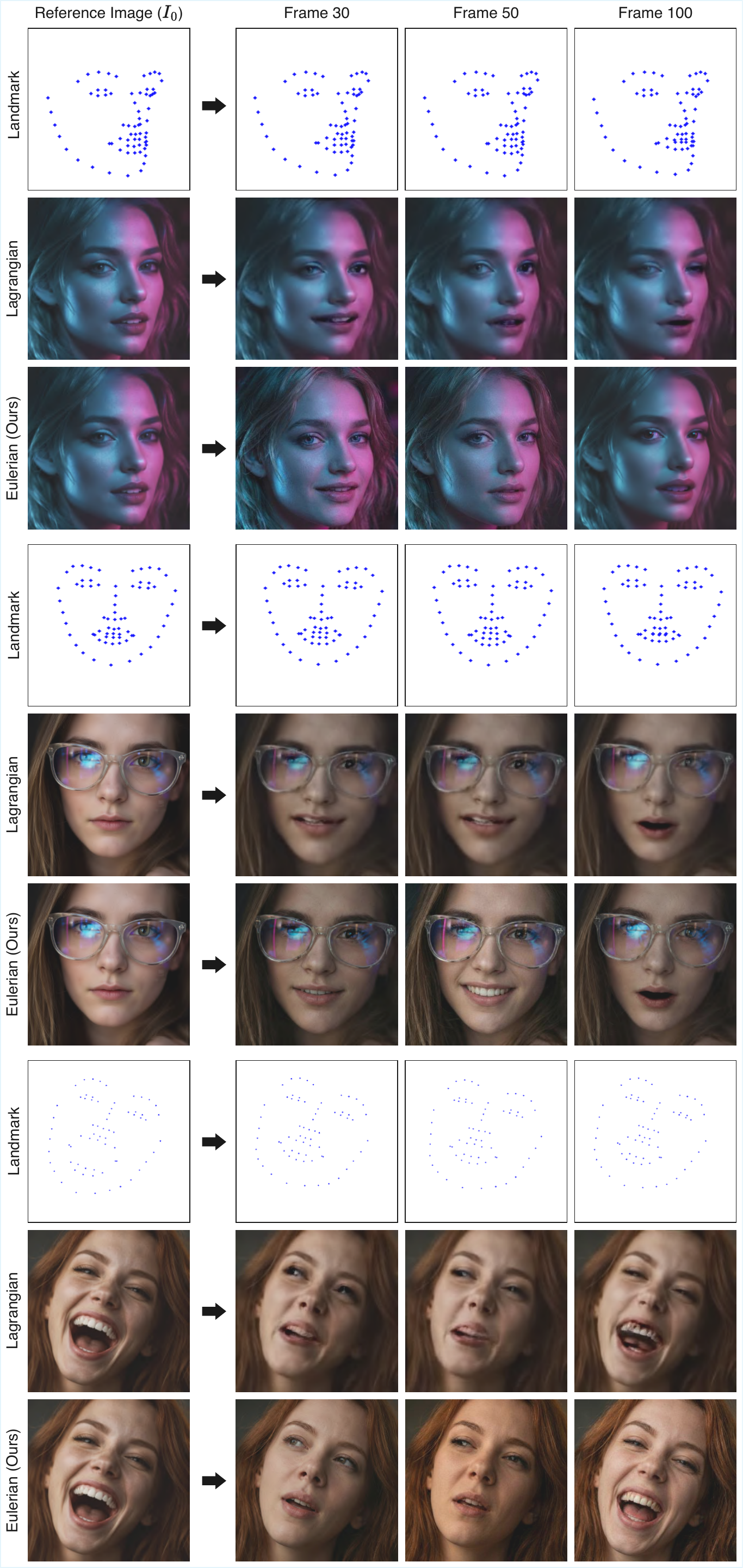}
  \caption{Extended Qualitative Evaluation on Landmark-Guided Portrait Animation. We compare our method against a standard Lagrangian baseline driven by the same sparse facial landmarks. (Top Group) Under complex lighting, the Lagrangian baseline loses skin texture detail and identity by Frame 100, whereas our method preserves the subject's likeness. (Middle Group) The Lagrangian approach struggles with rigid objects attached to the face; note the severe distortion of the eyeglasses in the later frames (Lagrangian) compared to the geometric stability maintained by our Eulerian method. (Bottom Group) In high-dynamic expressions (laughing), our method prevents the ``blurring'' of teeth and mouth interior often observed in reference warping.}
  \label{fig:appendix_more_visual_results_landmarks}
  \vspace{-15pt}
\end{figure}

\section{Sensitivity Analysis of Geometric Consistency}
\label{app:sensitivity_analysis}
As introduced in Section \ref{subsect:bgc}, our Bidirectional Geometric Consistency (BGC) mask relies on a dynamic threshold defined by $\alpha_{1}$ and $\alpha_{2}$ to identify valid correspondences. To validate the robustness of these hyperparameters and understand their interaction with motion magnitude and texture complexity, we conduct a comprehensive sensitivity analysis.

\noindent\textbf{Hyperparameter Robustness.} The threshold balances a dynamic motion-dependent term ($\alpha_{1}$) and a static noise floor ($\alpha_{2}$). In Table \ref{tab:sensitivity_alpha}, we vary $\alpha_{1} \in \{0.005, 0.01, 0.05\}$ and $\alpha_{2} \in \{0.1, 0.5, 1.0\}$ and report the resulting LPIPS and FVD on a 200-video subset of the WebVid test set. We observe that our framework is relatively stable across a moderate range of values. Setting $\alpha_{2}$ too low ($0.1$) makes the mask overly aggressive, discarding valid gradients and slightly degrading FVD, whereas setting $\alpha_{1}$ too high ($0.05$) makes the mask too permissive, allowing forward-backward drift to contaminate the training signal, which increases LPIPS. The empirical choice of $\alpha_{1}=0.01$ and $\alpha_{2}=0.5$ provides the optimal trade-off between structural preservation and training efficiency.

\begin{table}[h!]
  \centering
  \caption{\textbf{Sensitivity Analysis of BGC Thresholds.} Evaluation of LPIPS ($\downarrow$) and FVD ($\downarrow$) under varying $\alpha_{1}$ and $\alpha_{2}$ values. The model is highly robust near our default settings ($\alpha_{1}=0.01, \alpha_{2}=0.5$).}
  \label{tab:sensitivity_alpha}
  \resizebox{0.6\linewidth}{!}{
  \renewcommand{\arraystretch}{1.1}
  \begin{tabular}{cc|cc}
    \toprule
    $\alpha_{1}$ (Dynamic) & $\alpha_{2}$ (Static) & \textbf{LPIPS} $\downarrow$ & \textbf{FVD} $\downarrow$ \\
    \midrule
    0.005 & 0.5 & 0.1581 & 77.05 \\
    \textbf{0.010} & \textbf{0.5} & \textbf{0.1562} & \textbf{76.18} \\
    0.050 & 0.5 & 0.1634 & 78.42 \\
    \midrule
    0.010 & 0.1 & 0.1578 & 79.11 \\
    0.010 & 1.0 & 0.1610 & 77.85 \\
    \bottomrule
  \end{tabular}}
  \vspace{-10pt}
\end{table}

\noindent\textbf{Interaction with Motion Magnitude.} 
The dynamic formulation of the occlusion threshold is specifically designed to adapt to varying motion scales . In regions with extreme motion trajectories, the inherent interpolation variance of the flow estimator $\mathrm{\Phi}$ increases, leading to a naturally higher cycle error even in non-occluded regions. By scaling the threshold with the flow magnitude via $\alpha_{1} \| \mathbf{f} \|_{2}^{2}$, the mask automatically relaxes its strictness for fast-moving objects, preventing the erroneous masking of high-velocity foreground elements. As demonstrated in our trajectory-based section \ref{subsect:bgc}, this prevents the ``\textit{vanishing subject}'' artifacts commonly seen in fixed-threshold Eulerian approaches.

\noindent\textbf{Interaction with Texture Complexity.} 
The static threshold component, $\alpha_{2}$, acts as an essential noise floor for texture variations. In regions with high-frequency textures, e.g., foliage, intricate clothing, sub-pixel flow estimation becomes inherently noisy. Conversely, in textureless regions, e.g., plain walls, clear skies, the aperture problem causes flow ambiguity, leading to small but non-zero forward-backward cycle errors. The parameter $\alpha_{2}$ ensures that these minor, texture-induced flow variations are not incorrectly flagged as dis-occlusions. During our qualitative analysis, removing $\alpha_{2}$ resulted in a highly fragmented gradient mask, leading to localized blurring in textured regions due to intermittent supervision.

\end{document}